\documentclass[letterpaper, 10 pt, conference]{ieeeconf}  %

\IEEEoverridecommandlockouts                              %

\usepackage{times}
\usepackage{amssymb,amsfonts,amsmath,amscd}
\usepackage[pdftex]{graphicx}
\usepackage[noadjust]{cite} %
\usepackage{fancyhdr}
\usepackage[nolist]{acronym} %
\usepackage[ruled,vlined,linesnumbered]{algorithm2e}  %
\usepackage{placeins} %
\usepackage{microtype} %
\usepackage{setspace} %
\usepackage{paralist} %

\usepackage[usenames,dvipsnames]{color} %

\usepackage[pdftex,colorlinks,bookmarks=true]{hyperref} %

\newtheorem{thm}{Theorem}

\newtheorem{rem}{Remark}
\newtheorem{pdef}{Problem Definition}

\newenvironment{expl}{\noindent\hspace{2em}{\itshape Explanation: }}{\hfill\QEDopen}

\pdfoutput=1

\SetAlgoSkip{}

\let\oldnl\nl%
\newcommand{\skipln}{\renewcommand{\nl}{\let\nl\oldnl}} 

\newcommand{\algo}[1]{Alg.~\ref{#1}}
\newcommand{\algoline}[2]{\algo{#1}, Line~\ref{#2}}
\newcommand{\algolines}[3]{\algo{#1}, Lines~\ref{#2}--\ref{#3}}

\newcommand{\fullFigGap}[0]{\vspace{-1.5\baselineskip}}

\newcommand{\algorithmStyle}[0]{\footnotesize} %
\newcommand{\mathStyle}[0]{\textstyle}

\DeclareMathOperator*{\argmin}{arg\,min}

\newcommand{\norm}[2]{\left|\left| #1 \right|\right|_{#2}}
\newcommand{\card}[1]{\left|#1\right|}

\newcommand{\prob}[1]{P\left(#1\right)}
\newcommand{\uniform}[1]{\mathcal{U}\left(#1\right)}
\newcommand{\lebesgueSymb}[0]{\lambda}
\newcommand{\lebesgue}[1]{\lebesgueSymb\left(#1\right)}
\newcommand{\unitBall}[0]{\zeta_n}

\newcommand{\bbm}{\begin{bmatrix}}
\newcommand{\ebm}{\end{bmatrix}}

\newcommand{\pair}[2]{\left( #1, #2\right)}
\newcommand{\set}[1]{\left\lbrace #1\right\rbrace}

\newcommand{\setst}[2]{\left\lbrace #1\;\;\middle|\;\;#2\right\rbrace}

\newcommand{\setInsert}[0]{\xleftarrow{\scriptscriptstyle +}}
\newcommand{\setRemove}[0]{\xleftarrow{\scriptscriptstyle -}}

\newcommand{\stateSet}[0]{X}
\newcommand{\statex}[0]{\mathbf{x}}
\newcommand{\statey}[0]{\mathbf{y}}

\newcommand{\vertexSet}[0]{V}
\newcommand{\stateu}[0]{\mathbf{u}}
\newcommand{\statev}[0]{\mathbf{v}}
\newcommand{\statew}[0]{\mathbf{w}}
\newcommand{\edgeSet}[0]{E}

\newcommand{\treeGraph}[0]{\mathcal{T}}

\newcommand{\interp}[0]{t}
\newcommand{\pathCost}[0]{s}
\newcommand{\cost}[0]{c}
\newcommand{\costToCome}[0]{g}
\newcommand{\costToGo}[0]{h}
\newcommand{\solutionCost}[0]{f}
\newcommand{\edgeCost}[0]{\cost}
\newcommand{\queue}[0]{\mathcal{Q}}
\newcommand{\samplesPerBatch}[0]{m}

\newcommand{\totalSamples}[0]{q}

\newcommand{\pathSeq}[0]{\sigma}
\newcommand{\pathSet}[0]{\Sigma}
\newcommand{\bestPath}[0]{\pathSeq^{*}}
\newcommand{\bestPathCost}[0]{\pathCost^{*}}

\newcommand{\costFromBelow}[1]{\widehat{#1}}
\newcommand{\trueCost}[1]{#1}
\newcommand{\costFromAbove}[1]{#1_{\treeGraph}}
\newcommand{\gBelow}[1]{\costFromBelow{\costToCome}\left(#1\right)}
\newcommand{\gTrue}[1]{\trueCost{\costToCome}\left(#1\right)}
\newcommand{\gAbove}[1]{\costFromAbove{\costToCome}\left(#1\right)}
\newcommand{\hBelow}[1]{\costFromBelow{\costToGo}\left(#1\right)}

\newcommand{\fBelow}[1]{\costFromBelow{\solutionCost}\left(#1\right)}

\newcommand{\cBelow}[2]{\costFromBelow{\edgeCost}\left(#1,#2\right)}
\newcommand{\cTrue}[2]{\trueCost{\edgeCost}\left(#1,#2\right)}

\newcommand{\namedSet}[1]{\stateSet_{#1}}
\newcommand{\obsSet}[0]{\namedSet{\rm obs}}
\newcommand{\freeSet}[0]{\namedSet{\rm free}}
\newcommand{\goalSet}[0]{\namedSet{\rm goal}}
\newcommand{\nearSet}[0]{\namedSet{\rm near}}

\newcommand{\sampleSet}[0]{\namedSet{\rm samples}}
\newcommand{\fBelowSet}[0]{\namedSet{\costFromBelow{\solutionCost}}}

\newcommand{\namedState}[1]{\statex_{#1}}
\newcommand{\xstart}[0]{\namedState{\rm start}}
\newcommand{\xgoal}[0]{\namedState{\rm goal}}
\newcommand{\xmin}[0]{\namedState{\rm m}}

\newcommand{\namedVertex}[1]{\statev_{#1}}
\newcommand{\vmin}[0]{\namedVertex{\rm m}}
\newcommand{\bestEdge}[0]{\pair{\vmin}{\xmin}}

\newcommand{\namedCost}[1]{\cost_{#1}}
\newcommand{\cbest}[0]{\namedCost{\rm best}}
\newcommand{\cbestfinal}[0]{\namedCost{{\rm best},\totalSamples}}

\newcommand{\namedQueue}[1]{\queue_{#1}}
\newcommand{\edgeQueue}[0]{\namedQueue{E}}
\newcommand{\vertexQueue}[0]{\namedQueue{V}}

\newcommand{\namedVertexSet}[1]{\vertexSet_{#1}}
\newcommand{\oldVertices}[0]{\namedVertexSet{\rm old}}
\newcommand{\nearVertices}[0]{\namedVertexSet{\rm near}}

\newcommand{\UTIAStitle}{Batch Informed Trees (BIT*): Sampling-based Optimal Planning via the Heuristically Guided Search of Implicit Random Geometric Graphs}

\title{\LARGE \bf \UTIAStitle}
\author{Jonathan D.\ Gammell$^1$, Siddhartha S.\ Srinivasa$^2$, and Timothy D.\ Barfoot$^1$%
\thanks{$^1$ J.\ D.\ Gammell and T.\ D.\ Barfoot are with the Autonomous Space Robotics Lab at the University of Toronto, Toronto, Ontario, Canada.
Email: \texttt{\{jon.gammell, tim.barfoot\}@utoronto.ca}}
\thanks{$^2$ S.\ S.\ Srinivasa is with The Personal Robotics Lab at Carnegie Mellon University, Pittsburgh, Pennsylvania, USA. Email: \texttt{siddh@cs.cmu.edu}}%
}

\hypersetup{%
    pdftitle={Gammell et al.: \UTIAStitle},
    pdfauthor={Jonathan D. Gammell, Siddhartha S. Srinivasa, and Timothy D. Barfoot},
    pdfkeywords={},
    pdfsubject={},
    pdfstartview=FitH,%
    bookmarksopen=true,%
    breaklinks=true,%
    colorlinks=true,%
    linkcolor=blue,anchorcolor=blue,%
    citecolor=blue,filecolor=blue,%
    menucolor=blue,%
    urlcolor=blue
}%

\begin{document}
\begin{acronym}[UTIAS]

    \acro{ASRL}{Autonomous Space Robotics Lab}
    \acro{CMU}{Carnegie Mellon University}
    \acro{CSA}{Canadian Space Agency}
    \acro{DRDC}{Defence Research and Development Canada}
    \acro{KSR}{Koffler Scientific Reserve at Jokers Hill}
    \acro{MET}{Mars Emulation Terrain}
    \acro{MIT}{Massachusetts Institute of Technology}
    \acro{NASA}{National Aeronautics and Space Administration}
    \acro{NSERC}{Natural Sciences and Engineering Research Council of Canada}
    \acro{NCFRN}{\acs{NSERC} Canadian Field Robotics Network}
    \acro{NORCAT}{Northern Centre for Advanced Technology Inc.}
    \acro{ODG}{Ontario Drive and Gear Ltd.}
    \acro{ONR}{Office of Naval Research}
    \acro{USSR}{Union of Soviet Socialist Republics}
    \acro{UofT}{University of Toronto}
    \acro{UW}{University of Waterloo}
    \acro{UTIAS}{University of Toronto Institute for Aerospace Studies}

    \acro{ACPI}{advanced configuration and power interface}
    \acro{CLI}{command-line interface}
    \acro{GUI}{graphical user interface}
    \acro{JIT}{just-in-time}
    \acro{LAN}{local area network}
    \acro{MFC}{Microsoft foundation class}
    \acro{NIC}{network interface card}
    \acro{SDK}{software development kit}
    \acro{HDD}{hard-disk drive}
    \acro{SSD}{solid-state drive}

    \acro{ICRA}{IEEE International Conference on Robotics and Automation}
    \acro{IJRR}{International Journal of Robotics Research}
    \acro{IROS}{IEEE/RSJ International Conference on Intelligent Robots and Systems}
    \acro{RSS}{Robotics: Science and Systems Conference}

    \acro{DOF}{degree-of-freedom}
        \acrodefplural{DoF}[DoFs]{degrees-of-freedom} %

        \acro{FOV}{field of view}
            \acrodefplural{FOV}[FOVs]{fields of view}
        \acro{HDOP}{horizontal dilution of position}
        \acro{UTM}{universal transverse mercator}
        \acro{WAAS}{wide area augmentation system}
        \acro{AHRS}{attitude heading reference system}
        \acro{DAQ}{data acquisition}
        \acro{DGPS}{differential global positioning system}
        \acro{DPDT}{double-pole, double-throw}
        \acro{DPST}{double-pole, single-throw}
        \acro{GPR}{ground penetrating radar}
        \acro{GPS}{global positioning system}
        \acro{LED}{light-emitting diode}
        \acro{IMU}{inertial measurement system}
        \acro{PTU}{pan-tilt unit}
        \acro{RTK}{real-time kinematic}
        \acro{R/C}{radio control}
        \acro{SCADA}{supervisory control and data acquisition}
        \acro{SPST}{single-pole, single-throw}
        \acro{SPDT}{single-pole, double-throw}
        \acro{UWB}{ultra-wide band}

    \acro{DDS}{Departmental Doctoral Seminar}
    \acro{DEC}{Doctoral Examination Committee}
    \acro{FOE}{Final Oral Exam}
    \acro{ICD}{Interface Control Document}

    \acro{iid}[i.i.d.]{independent and identically distributed}
    \acro{aas}[a.a.s.]{asymptotically almost-surely}
    \acro{RGG}{random geometric graph}

    \acro{EKF}{extended Kalman filter}
    \acro{iSAM}{incremental smoothing and mapping}
    \acro{ISRU}{in-situ resource utilization}
    \acro{PCA}{principle component analysis}
    \acro{SLAM}{simultaneous localization and mapping}
    \acro{SVD}{singular value decomposition}
    \acro{UKF}{unscented Kalman filter}
    \acro{VO}{visual odometry}
    \acro{VTR}[VT\&R]{visual teach and repeat}

        \acro{ADstar}[AD*]{Anytime D∗}
        \acro{ADAstar}[ADA*]{Anytime Dynamic A*}
        \acro{ARAstar}[ARA*]{Anytime Repairing A*}
        \acro{BITstar}[BIT*]{Batch Informed Trees}
            \acrodefplural{BITstar}[BIT*]{Batch Informed Trees}
        \acro{BRM}{belief roadmap}
        \acro{Dstar}[D*]{Dynamic A∗}
        \acro{EST}{Expansive Space Tree}
            \acrodefplural{EST}[EST]{Expansive Space Trees}
        \acro{FMTstar}[FMT*]{Fast Marching Trees}
            \acrodefplural{FMTstar}[FMT*]{Fast Marching Trees}
        \acro{LQG-MP}{linear-quadratic Gaussian motion planning}
        \acro{LPAstar}[LPA*]{Lifelong Planning A*}
        \acro{MPLB}{Motion Planning Using Lower Bounds}
        \acro{MDP}{Markov decision process}
            \acrodefplural{MDP}[MDPs]{Markov decision processes}
        \acro{NRP}{network of reusable paths}
            \acrodefplural{NRP}[NRPs]{networks of reusable paths}
        \acro{POMDP}{partially-observable Markov decision process}
            \acrodefplural{POMDP}[POMDPs]{partially-observable Markov decision processes}
        \acro{PRM}{Probabilistic Roadmap}
            \acrodefplural{PRM}[PRM]{Probabilistic Roadmaps}
        \acro{PRMstar}[PRM*]{optimal \acp{PRM}}
        \acro{RAstar}[RA*]{Randomized A*}
        \acro{RRG}{Rapidly-exploring Random Graph}
        \acro{RRM}{Rapidly-exploring Roadmap}
        \acro{RRT}{Rapidly-exploring Random Tree}
            \acrodefplural{RRT}[RRT]{Rapidly-exploring Random Trees}
        \acro{hRRT}{Heuristically Guided \acs{RRT}}
        \acro{RRTstar}[RRT*]{optimal \acs{RRT}}
            \acrodefplural{RRTstar}[RRT*]{optimal \acsp{RRT}}
        \acro{RRBT}{rapidly-exploring random belief tree}
        \acro{SBAstar}[SBA*]{Sampling-based A*}

    \acro{HERB}{Home Exploring Robot Butler}
    \acro{MER}{Mars Exploration Rover}
    \acro{MSL}{Mars Science Laboratory}
    \acro{OMPL}{Open Motion Planning Library}
    \acro{ROS}{Robot Operating System}

\end{acronym} %

\maketitle

\thispagestyle{empty}
\pagestyle{empty}
\setlength\parskip{0ex plus0.1ex minus0.1ex}

\makeatletter
\renewcommand\section{\@startsection{section}{1}{\z@}%
    {0.8ex plus 0.0ex minus 0.3ex}%
    {0.5ex plus 0.0ex minus 0ex}%
    {\normalfont\normalsize\centering\scshape}}%
\renewcommand\subsection{\@startsection{subsection}{2}{\z@}%
    {0.8ex plus 0.0ex minus 0.3ex}%
    {0.5ex plus 0.0ex minus 0ex}%
    {\normalfont\normalsize\itshape}}%
\renewcommand\subsubsection{\@startsection{subsubsection}{3}{\parindent}%
    {0.5ex plus 0.0ex minus 0.1ex}%
    {0ex}%
    {\normalfont\normalsize\itshape}}%
\renewcommand\paragraph{\@startsection{paragraph}{4}{2\parindent}%
    {0ex plus 0.0ex minus 0.1ex}%
    {0ex}%
    {\normalfont\normalsize\itshape}}%
\makeatother

\setlength \abovedisplayskip{1ex plus0pt minus1pt}
\setlength \belowdisplayskip{1ex plus0pt minus1pt}

\setlength{\skip\footins}{0.5\baselineskip  plus 0.0\baselineskip  minus 0.2\baselineskip} %

\setlength\floatsep{0.9\baselineskip plus0pt minus0.2\baselineskip}                      %
\setlength\textfloatsep{0.2\baselineskip plus0pt minus0.4\baselineskip}                %
\setlength\abovecaptionskip{0.pt plus0pt minus0pt}                                      %

\begin{abstract}
In this paper, we present \ac{BITstar}, a planning algorithm based on unifying graph- and sampling-based planning techniques.
By recognizing that a set of samples describes an implicit \ac{RGG}, we are able to combine the efficient ordered nature of graph-based techniques, such as A*, with the anytime scalability of sampling-based algorithms, such as \acp{RRT}.

\ac{BITstar} uses a heuristic to efficiently search a series of increasingly dense implicit \acp{RGG} while reusing previous information.
It can be viewed as an extension of incremental graph-search techniques, such as \ac{LPAstar}, to continuous problem domains as well as a generalization of existing sampling-based optimal planners.
It is shown that it is probabilistically complete and asymptotically optimal.

We demonstrate the utility of \ac{BITstar} on simulated random worlds in $\mathbb{R}^2$ and $\mathbb{R}^8$ and manipulation problems on \acs{CMU}'s \acs{HERB}, a 14-\acs{DOF} two-armed robot.
On these problems, \ac{BITstar} finds better solutions faster than \ac{RRT}, \acs{RRTstar}, Informed \acs{RRTstar}, and \ac{FMTstar} with faster anytime convergence towards the optimum, especially in high dimensions.
\end{abstract}
\acresetall %

\section{Introduction}\label{sec:intro}
Graph-search and sampling-based methods are two popular techniques for path planning in robotics.
Graph-based searches, such as Dijkstra's algorithm \cite{dijkstra_59} and A* \cite{hart_tssc68}, use dynamic programming \cite{bellman_ams54} to exactly solve a discrete approximation of a problem. %
These algorithms are not only \emph{resolution complete} but also \emph{resolution optimal}, always finding the optimal solution to the given problem at the chosen discretization, if one exists.
A* does this efficiently by using a heuristic to estimate the total cost of a solution constrained to pass through a state.
The result is an algorithm that searches in order of decreasing solution quality and is \emph{optimally efficient}. %
Any other optimal algorithm using the same heuristic will expand at least as many vertices as A* \cite{hart_tssc68}.

The quality of the \emph{continuous} solution found by these graph-search techniques depends heavily on the discretization of the problem.
Finer discretization increases the quality of the solution \cite{bertsekas_tac75}, but also increases the computational effort necessary to find it.
This becomes a significant problem in high-dimensional spaces, such as for manipulation planning (Fig.~\ref{fig:2armHerb}), as the size of the discrete state space grows exponentially with the number of dimensions.
Bellman \cite{bellman_57} referred to this problem as the \textit{curse of dimensionality}.
Graph-search techniques have still been successful as planning algorithms \cite{lozano-perez_cacm79}  on a variety of graph types \cite{chen_icra92,sallaberger_acta95}, including for nonholonomoic robots \cite{barraquand_icra91,lynch_ijrr96}, kinodynamic planning \cite{cherif_icra99,donald_jacm93}, and manipulation planning \cite{kondo_tra91}.

Graph search has also been extended to \emph{anytime} and \emph{incremental} search.
Anytime techniques \cite{likhachev_icaps05,ferguson_icra05,likhachev_ai08} quickly find a suboptimal path before completing the search for the optimum, while incremental techniques \cite{koenig_ai04,stentz_ijcai95,koenig_tro05,ferguson_icra05,likhachev_ai08} handle changes in a graph efficiently by reusing information.

\begin{figure}[t]
    \centering
    \includegraphics[width=\columnwidth]{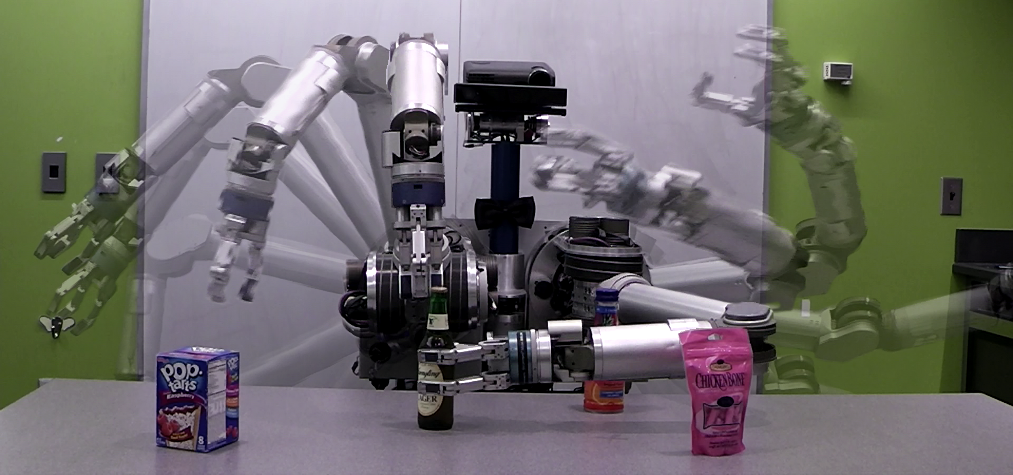}
    \caption{A composite figure of a trajectory generated by \acs{BITstar} for a difficult $14$-\acs{DOF} two-arm manipulation planning problem on \acs{HERB}.
        In the trial pictured, \acs{BITstar} found a solution in $4$ seconds and spent $2.5$ minutes refining it.
        Over $25$ trials with $2.5$ minutes of computational time, \acs{BITstar} had a median solution cost of $17.4$ and success rate of $68\%$, while Informed \acs{RRTstar} and \acs{FMTstar} had median costs of $25.3$ and $17.2$ and success rates of $8\%$ and $36\%$, respectively.
        Nonoptimal planners, \acs{RRT} and \acs{RRT}-Connect, had median costs of $31.1$ and $22.1$ and success rates of $8\%$ and $100\%$, respectively.}
    \label{fig:2armHerb}
\end{figure}%

Sampling-based planners, such as \acp{PRM} \cite{kavraki_tro96}, \acp{RRT} \cite{lavalle_ijrr01}, and \acp{EST} \cite{hsu_ijrr02}, avoid the discretization problems of graph-search techniques by randomly sampling the continuous planning domain.
This scales more effectively to high-dimensional problems, but makes their search probabilistic.
They are \emph{probabilistically complete}, having a probability of finding a solution, if one exists, that goes to one as the number of samples goes to infinity.
Anytime algorithms, such as \ac{RRT} and \ac{EST}, also have \emph{anytime resolution}, a growing representation of the problem domain that becomes increasingly accurate as the number of iterations increases.
Optimal variants, such as \acs{RRTstar} and \acs{PRMstar} \cite{karaman_ijrr11}, are also \emph{asymptotically optimal}, converging asymptotically to the optimal solution with probability one as the number of samples goes to infinity (\emph{almost sure} asymptotic convergence).\acused{RRTstar}\acused{PRMstar}
While solutions improve with computational time, this does not guarantee a reasonable rate of convergence as the random sampling is inherently \emph{unordered}.

There is a long history of adding graph-search concepts to sampling-based planners.
Algorithms have used heuristics to refine the \ac{RRT} search, including by biasing the sampling procedure \cite{urmson_iros03}, and to define a series of subplanning problems given the current solution \cite{ferguson_iros06}.
Similarly, focusing techniques have also been used to limit the search of \ac{RRTstar} once it finds a solution \cite{akgun_iros11,otte_tro13,gammell_iros14}.
While these techniques can improve the initial solution and/or the convergence rate to the optimum, their \acs{RRT}-based search is still unordered.

Other algorithms order the search at the expense of anytime resolution.
\ac{FMTstar} \cite{janson_isrr13} uses a marching method to process a single set of samples.
The resulting search is ordered on cost-to-come but must be restarted if a higher resolution is needed.
The \ac{MPLB} algorithm \cite{salzman_icra15} extends \ac{FMTstar} to quasi-anytime resolution and an ordering given by estimating the cost of solutions constrained to pass through each state.
The quasi-anytime resolution is achieved by solving a series of independent problems with an increasing number of samples.
It is stated that this can be done efficiently by reusing information, but no specific methods are presented.

Still other algorithms attempt to extend graph-search directly to continuous planning problems.
In \ac{RAstar} \cite{diankov_iros07} and \ac{SBAstar} \cite{persson_ijrr14} a tree is grown towards solutions by sampling near heuristically selected vertices.
This biases the growth of the tree towards good solutions but requires methods to avoid local minima.
\ac{RAstar} defines a minimum-allowed distance between vertices, limiting the number of times a vertex can be expanded but also limiting the final resolution.
\ac{SBAstar} includes a measure of local sample density in the vertex expansion heuristic.
This decreases the priority of sampling near frequently expanded vertices, but requires methods to estimate local sample density.

In this paper, we present \ac{BITstar}, a planning algorithm that balances the benefits of graph-search and sampling-based techniques.
It uses batches of samples to perform an ordered search on a continuous planning domain while maintaining anytime performance.
By processing samples in batches, its search can be ordered around the minimum solution proposed by a heuristic, as in A* \cite{hart_tssc68}.
By processing multiple batches of samples, it converges asymptotically towards the global optimum with anytime resolution, as in \ac{RRTstar} \cite{karaman_ijrr11}.
This is done efficiently by using incremental search techniques to incorporate the new samples into the existing search, as in \ac{LPAstar} \cite{koenig_ai04}.
The multiple batches also allow subsequent searches to be focused on the subproblem that could contain a better solution, as in Informed \ac{RRTstar} \cite{gammell_iros14}.

The performance of \ac{BITstar} is demonstrated both on random experiments in $\mathbb{R}^2$ and $\mathbb{R}^{8}$ and manipulation problems on the \acs{CMU} Personal Robotic Lab's \ac{HERB} \cite{herb}.
The results show that \ac{BITstar} consistently outperformed both nonasymptotically and asymptotitcally optimal planners (\ac{RRT}, \ac{RRTstar}, Informed \acs{RRTstar}, and \ac{FMTstar}).
It was more likely to have found a solution at a given computational time and converged towards the optimum faster.
The same held in difficult planning problems on \ac{HERB}, where collision checking is expensive.
\ac{BITstar} was nearly twice as likely to find a solution to a difficult two-arm problem (Fig.~\ref{fig:2armHerb}) and found better solutions on easier one-arm problems (Fig.~\ref{fig:1armHerb}).
The only planner tested that found solutions faster was \acs{RRT}-Connect, which does not converge towards the optimum.

The remainder of this paper is organized as follows.
Section~\ref{sec:back} presents further background and Section~\ref{sec:bit} presents a description of the algorithm.
Section~\ref{sec:anal} presents an initial theoretical analysis of \ac{BITstar}, while Section \ref{sec:exp} presents the experimental results in detail.
Finally, Section~\ref{sec:disc} presents a discussion on the algorithm and related future work and Section~\ref{sec:conc} provides a conclusion.

\begin{figure}[t]
    \centering
    \includegraphics[page=1,width=\columnwidth]{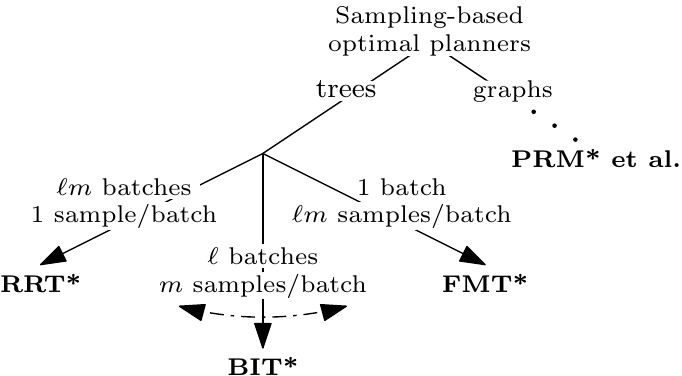}
    \caption{A simplified taxonomy of sampling-based optimal planners demonstrating the relationship between \acs{RRTstar}, \acs{FMTstar}, and \acs{BITstar}.}
    \label{fig:taxonomy}
\end{figure}%
\begin{figure*}[t]
    \centering
    \includegraphics[width=\textwidth]{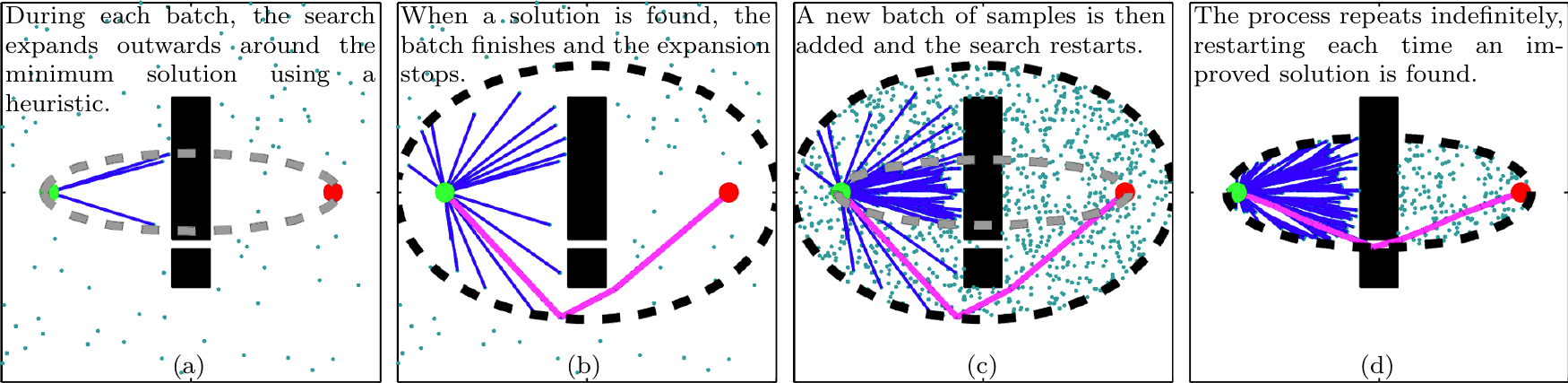}
    \caption{An illustration of the informed search procedure used by \acs{BITstar}.
        The start and goal states are shown as green and red, respectively.
        The current solution is highlighted in magenta.
        The subproblem that contains any better solutions is shown as a black dashed line, while the progress of the current batch is shown as a grey dashed line.
        Fig.~(a) shows the growing search of the first batch of samples, and (b) shows the first search ending when a solution is found.
        After pruning and adding a second batch of samples, Fig.~(c) shows the search restarting on a denser graph while (d) shows the second search ending when an improved solution is found.
        An animated illustration is available in the attached video.
        \fullFigGap}
    \label{fig:cartoon}
\end{figure*}%
\section{Background}\label{sec:back}
We define the optimal planning problem similarly to \cite{karaman_ijrr11}.
\begin{pdef}[Optimal Planning]\label{pdef:opt}
    Let $\stateSet \subseteq \mathbb{R}^n$ be the state space of the planning problem, $\obsSet \subset \stateSet$ be the states in collision with obstacles, and $\freeSet = \stateSet \setminus \obsSet$ be the resulting set of permissible states.
    Let $\xstart \in \freeSet$ be the initial state and $\goalSet \subset \freeSet$ be the set of desired final states.
    Let $\pathSeq : \; \left[0,1\right] \mapsto \stateSet$ be a sequence of states (a path) and $\pathSet$ be the set of all nontrivial paths.

    The optimal solution is the path, $\bestPath$, that minimizes a chosen cost function, $\pathCost : \; \pathSet \mapsto \mathbb{R}_{\geq0}$, while connecting $\xstart$ to any $\xgoal \in \goalSet$ through free space,
    \begin{align*}
        \mathStyle
        \bestPath = \argmin\limits_{\pathSeq \in \pathSet} \left\lbrace \pathCost \left(\pathSeq\right) \;\; \middle| \;\; \right. & \pathSeq(0) = \xstart,\, \pathSeq(1) \in \xgoal,\\
        & \quad \left. \forall \interp \in \left[ 0,1 \right],\, \pathSeq\left(\interp\right) \in \freeSet \right\rbrace,
    \end{align*}%
    where $\mathbb{R}_{\geq0}$ is the set of non-negative real numbers.
    We denote the cost of this optimal path as $\bestPathCost$.\hfill\QEDopen
\end{pdef}%

A discrete set of states in this state space, $\sampleSet \subset \stateSet$, can be viewed as a graph whose edges are given algorithmically by a transition function (an \emph{implicit} graph).
When these states are sampled randomly, $\sampleSet = \set{\statex \sim \uniform{\stateSet}}$, the properties of the graph can be described by a probabilistic model known as a \ac{RGG} \cite{penrose_03}.

In an \ac{RGG}, the connections (edges) between states (vertices) depend on their relative geometric position.
Common \acp{RGG} have edges to a specific number of each state's nearest neighbours (a $k$-nearest graph \cite{xue_wire04}) or to all neighbours within a specific distance (an $r$-disc graph \cite{gilbert_siam61}).
\ac{RGG} theory provides probabilistic relationships between the number and distribution of samples, the $k$ or $r$ defining the graph, and specific graph properties such as connectivity or relative cost through the graph \cite{penrose_03,muthukrishnan_siam05,karaman_ijrr11,janson_isrr13}.

Sampling-based planners can therefore be viewed as algorithms to construct an implicit \ac{RGG} and an explicit spanning tree in the free space of the planning problem.
Much like graph-search techniques, the performance of an algorithm will depend on the quality of the \ac{RGG} representation and the efficiency of the search.

Karaman and Frazzoli \cite{karaman_ijrr11} use \ac{RGG} theory in \ac{RRTstar} to limit graph complexity while maintaining probabilistic bounds on the representation, but the graph is constructed and searched simultaneously, resulting in a randomly ordered anytime search.
Janson and Pavone \cite{janson_isrr13} similarly use \ac{RGG} theory in \ac{FMTstar}, but for a constant number of samples, resulting in an ordered but nonanytime (in solution or resolution) search.
Recently, Salzman and Halperin \cite{salzman_icra15} have given \ac{FMTstar} quasi-anytime performance by independently solving increasingly dense \acp{RGG} in their \ac{MPLB} algorithm.
Heuristics order and focus the search, but solutions are only returned when an \ac{RGG} is completely searched.

\vspace{0.5ex}
In contrast, \ac{BITstar} uses \emph{incremental} search techniques on increasingly dense \acp{RGG}.
This balances the benefits of heuristically ordered search with anytime performance and asymptotic optimality.
The tuning parameters are the choice of the heuristic, an \ac{RGG} constant, and the number of samples per batch.
\ac{BITstar} can be viewed as an extension of \ac{LPAstar} \cite{koenig_ai04} to continuous problems and as a generalization of existing sampling-based optimal planners (Fig.~\ref{fig:taxonomy}).
With batches of one sample, it is a version of Informed \acs{RRTstar} \cite{gammell_iros14}, and with a single batch and the zero heuristic, a version of \ac{FMTstar}.

\section{Batch Informed Trees (\acs{BITstar})}\label{sec:bit}
Informally, \ac{BITstar} works as follows.
An initial \ac{RGG} with \emph{implicit} edges is defined by uniformly distributed random samples from the free space and the start and goal.
The \ac{RGG} parameter ($r$ or $k$) is chosen to reduce graph complexity while maintaining asymptotic optimality requirements as a function of the number of samples \cite{karaman_ijrr11,janson_isrr13}.
An \emph{explicit} tree is then built outwards from the start towards the goal by a heuristic search (Fig.~\ref{fig:cartoon}a).
This tree includes only collision-free edges and its construction stops when a solution is found or it can no longer be expanded (Fig.~\ref{fig:cartoon}b).
This concludes a \emph{batch}.

To start a new batch, a denser implicit \ac{RGG} is constructed by adding more samples and updating $r$ (or $k$).
If a solution has been found, these samples are limited to the subproblem that could contain a better solution (e.g., an ellipse for path length \cite{gammell_iros14}).
The tree is then updated using \acs{LPAstar}-style incremental search techniques that reuse existing information (Fig.~\ref{fig:cartoon}c).
As before, the construction of the tree stops when the solution cannot be improved or when there are no more collision-free edges to traverse (Fig.~\ref{fig:cartoon}d).
The process continues with new batches as time allows.

\subsection{Notation}\label{sec:bit:note}
The functions $\gBelow{\statex}$ and $\hBelow{\statex}$ represent admissible estimates of the cost-to-come to a state, $\statex \in \stateSet$, from the start and the cost-to-go from a state to the goal, respectively (i.e., they bound the true costs from below).
The function, $\fBelow{\statex}$, represents an admissible estimate of the cost of a path from $\xstart$ to $\goalSet$ constrained to pass through $\statex$, i.e., $\fBelow{\statex} := \gBelow{\statex} + \hBelow{\statex}$.
This estimate defines a subset of states, $\fBelowSet := \setst{\statex\in\stateSet}{\fBelow{\statex} \leq \cbest}$, that could provide a solution better than the current best solution cost, $\cbest$.

Let $\treeGraph := \pair{\vertexSet}{\edgeSet}$ be an \emph{explicit} tree with a set of vertices, $\vertexSet\subset\freeSet$, and edges, $\edgeSet = \set{\pair{\statev}{\statew}}$ for some $\statev,\, \statew \in \vertexSet$.
The function $\gAbove{\statex}$ represents the cost-to-come to a state $\statex \in \stateSet$ from the start vertex given the current tree, $\treeGraph$.
We assume a state not in the tree, or otherwise unreachable from the start, has a cost-to-come of infinity.
It is important to recognize that these two functions will always bound the unknown true optimal cost to a state, $\gTrue{\cdot}$, i.e., $\forall \statex \in \stateSet,\,\gBelow{\statex} \leq \gTrue{\statex} \leq \gAbove{\statex}$.

The functions $\cBelow{\statex}{\statey}$ and $\cTrue{\statex}{\statey}$ represent an admissible estimate of the cost of an edge and the true cost of an edge between states $\statex,\, \statey \in \stateSet$, respectively.
We assume that edges that intersect the obstacle set have a cost of infinity, and therefore $\forall \statex,\, \statey \in \stateSet,\,\, \cBelow{\statex}{\statey} \leq  \cTrue{\statex}{\statey} \leq \infty$.
It is important to recognize that calculating $\cTrue{\statex}{\statey}$ can be expensive (e.g., collision detection, differential constraints, etc.) and using a heuristic estimate for edge cost has the effect of delaying this calculation until necessary.

The function $\lebesgue{\cdot}$ represents the Lebesgue measure of a set (e.g., the \textit{volume}), and $\unitBall$ represent the Lebesgue measure of an $n$-dimensional unit ball.
The cardinality of a set is denoted by $\card{\cdot}$.
We use the notation $\stateSet \setInsert \set{\statex} $ and $\stateSet \setRemove \set{\statex}$ to compactly represent the compounding operations $\stateSet \gets \stateSet \cup \set{\statex}$ and $\stateSet \gets \stateSet \setminus \set{\statex}$, respectively.
As is customary, we take the minimum of an empty set to be infinity.

\begin{algorithm}[t]
    \caption{\acs{BITstar}\algorithmStyle$\left(\xstart \in \freeSet, \xgoal \in \goalSet\right)$}\label{algo:bitstar}
    \algorithmStyle
    $\vertexSet \gets \set{\xstart};\;$
    $\edgeSet \gets \emptyset;\;$
    $\sampleSet \gets \set{\xgoal}$\label{algo:bitstar:initStart}\;
    $\edgeQueue \gets \emptyset;\;$
    $\vertexQueue \gets \emptyset;\;$
    $r \gets \infty$\label{algo:bitstar:initEnd}\;
    \Repeat{$\mathtt{STOP}$}
    {
        \If{$\edgeQueue \equiv \emptyset \; \mathbf{and} \; \vertexQueue \equiv \emptyset$\label{algo:bitstar:startBatch}}
        {
            $\mathtt{Prune}\left(\gAbove{\xgoal}\right)$\label{algo:bitstar:prune}\;
            $\sampleSet \setInsert \mathtt{Sample}\left(\samplesPerBatch, \gAbove{\xgoal}\right)$\label{algo:bitstar:addSamples}\;
            $\oldVertices \gets \vertexSet$\label{algo:bitstar:oldVertices}\;
            $\vertexQueue \gets \vertexSet$\label{algo:bitstar:resetVertexQueue}\;
            $r \gets \mathtt{radius}\left(\card{\vertexSet} + \card{\sampleSet}\right)$\label{algo:bitstar:r}\label{algo:bitstar:endBatch}\;
        }
        
        \While{$\mathtt{BestQueueValue}\left(\vertexQueue\right)\leq\mathtt{BestQueueValue}\left(\edgeQueue\right)$\label{algo:bitstar:startExpand}}
        {
            $\mathtt{ExpandVertex}\left( \mathtt{BestInQueue}\left(\vertexQueue\right) \right)$\label{algo:bitstar:updateEdgeQueue}\label{algo:bitstar:endExpand}\;
        }

        $\bestEdge \gets \mathtt{BestInQueue}\left(\edgeQueue\right)$\label{algo:bitstar:popStart}\;
        $\edgeQueue \setRemove \set{\bestEdge}$\label{algo:bitstar:popEnd}\;
        \If{$\gAbove{\vmin}+\cBelow{\vmin}{\xmin}+\hBelow{\xmin}<\gAbove{\xgoal}$\label{algo:bitstar:canBeBetterSoln}\label{algo:bitstar:startProcess}}
        {
            \If{$\gBelow{\vmin}+\cTrue{\vmin}{\xmin}+\hBelow{\xmin}<\gAbove{\xgoal}$\label{algo:bitstar:canRealEdgeEverBeBetterSoln}}
            {
                \If{$\gAbove{\vmin}+\cTrue{\vmin}{\xmin}<\gAbove{\xmin}$\label{algo:bitstar:improvesGraph}}
                {
                    \If{$\xmin \in \vertexSet$\label{algo:bitstar:isInTree}}
                    {
                        $\edgeSet \setRemove \set{\pair{\statev}{\xmin}\in\edgeSet}$\label{algo:bitstar:rmParent}\;
                    }
                    \Else
                    {
                        $\sampleSet \setRemove \set{\xmin}$\label{algo:bitstar:rmSample}\;
                        $\vertexSet \setInsert \set{\xmin};\;$
                        $\vertexQueue \setInsert \set{\xmin}$\label{algo:bitstar:addVertex}\;
                    }
                    $\edgeSet \setInsert \set{\bestEdge}$\label{algo:bitstar:addEdge}\;
                    $\edgeQueue \setRemove \left\lbrace\pair{\statev}{\xmin}\in\edgeQueue \;\;\middle|\vphantom{\gAbove{\statev} + \cBelow{\statev}{\xmin} \geq \gAbove{\xmin}}\right.$\label{algo:bitstar:pruneEdges}\\
                    \skipln$\qquad\qquad\qquad\; \left. \vphantom{\pair{\statev\in\vertexSet}{\xmin}\in\edgeQueue}\gAbove{\statev} + \cBelow{\statev}{\xmin} \geq \gAbove{\xmin}\right\rbrace$\;
                }
            }
        }
        \Else
        {
            $\edgeQueue \gets \emptyset;\;$
            $\vertexQueue \gets \emptyset$\label{algo:bitstar:clearQueue}\label{algo:bitstar:endProcess}\;
        }
    }
    \Return{$\treeGraph$}\;
\end{algorithm}%

\subsection{Algorithm}\label{sec:bit:detail}
\ac{BITstar} is presented in Algs.~\ref{algo:bitstar}--\ref{algo:prune}.
For simplicity, we limit our discussion to a search from the start to a single goal state using an $r$-disc \ac{RGG}, but the formulation is similar for searches from a goal state, with a goal set, or with a $k$-nearest \ac{RGG}.
The algorithm starts with a given initial state, $\xstart$, in the tree, $\treeGraph$, and the goal state, $\xgoal$, in the set of unconnected samples, $\sampleSet$ (\algoline{algo:bitstar}{algo:bitstar:initStart}).
The tree is grown towards $\xgoal$ from $\xstart$ by processing a queue of \ac{RGG} edges, $\edgeQueue$.
This edge queue is populated by a vertex expansion queue, $\vertexQueue$ (\algoline{algo:bitstar}{algo:bitstar:initEnd}).

\subsubsection{Batch creation (\algolines{algo:bitstar}{algo:bitstar:startBatch}{algo:bitstar:endBatch})}\label{sec:bit:detail:init}
A new batch begins when the queues are empty.
The samples and spanning tree are pruned of states that cannot improve the solution (\algoline{algo:bitstar}{algo:bitstar:prune}; \algo{algo:prune}).
A new set of $\samplesPerBatch$ samples is then added to the \ac{RGG} from the subproblem containing a better solution (\algoline{algo:bitstar}{algo:bitstar:addSamples}).
This can be accomplished by rejection sampling or, for some cost functions, direct sampling \cite{gammell_iros14}.
The vertices in the tree are labelled so that only connections to new states will be considered (\algoline{algo:bitstar}{algo:bitstar:oldVertices}) and requeued for expansion (\algoline{algo:bitstar}{algo:bitstar:resetVertexQueue}).
The radius of the underlying $r$-disc \ac{RGG} is updated to reflect its size, $q$, (\algoline{algo:bitstar}{algo:bitstar:r}),
\begin{align}\label{eqn:r}
    \mathStyle
        \mathtt{radius}\left(\totalSamples\right) := 2 \eta \left(1 + \frac{1}{n}\right)^{\frac{1}{n}} \left(\frac{\lebesgue{\fBelowSet}}{\unitBall}\right)^{\frac{1}{n}} \left(\frac{\log\left(\totalSamples\right)}{\totalSamples}\right)^{\frac{1}{n}},
\end{align}%
where $\eta \geq 1$ is a tuning parameter \cite{karaman_ijrr11}.

\begin{algorithm}[t]
    \caption{ExpandVertex\algorithmStyle$\left(\statev \in \vertexQueue \subseteq \vertexSet\right)$}\label{algo:expandVertex}
    \algorithmStyle
    $\vertexQueue \setRemove \set{\statev}$\label{algo:expandVertex:rmVertexQueue}\;
    $\nearSet \gets \setst{\statex \in \sampleSet}{\norm{\statex - \statev }{2} \leq r}$\label{algo:expandVertex:nearSamples}\;
    $\edgeQueue \setInsert \left\lbrace \pair{\statev}{\statex} \in \vertexSet\times \nearSet \;\; \middle| \;\; \vphantom{\gBelow{\statev} + \cBelow{\statev}{\statex} + \hBelow{\statex} < \gAbove{\xgoal}}\right.$\label{algo:expandVertex:insert}\\
    \skipln$\qquad\qquad\qquad\qquad\quad\left.\vphantom{\edgeQueue \setInsert \pair{\statev \in \vertexSet}{\statex \in \nearSet}} \gBelow{\statev} + \cBelow{\statev}{\statex} + \hBelow{\statex} < \gAbove{\xgoal} \right\rbrace$\;
    
    \If{$\statev\not\in\oldVertices$\label{algo:expandVertex:isNew}}
    {
        $\nearVertices \gets \setst{\statew \in \vertexSet}{\norm{\statew - \statev }{2} \leq r}$\label{algo:expandVertex:nearVertices}\;
        $\edgeQueue \setInsert \left\lbrace\pair{\statev}{\statew} \in \vertexSet \times \nearVertices \;\; \big| \;\; \pair{\statev}{\statew} \not\in \edgeSet, \vphantom{\gBelow{\statev} + \cBelow{\statev}{\statew} + \hBelow{\statew} < \gAbove{\xgoal},\gAbove{\statev} + \cBelow{\statev}{\statew} + < \gAbove{\statew}}\right.$\label{algo:expandVertex:insertRewire}\\
        \skipln$\qquad\qquad\qquad\qquad\gBelow{\statev} + \cBelow{\statev}{\statew} + \hBelow{\statew} < \gAbove{\xgoal},$\\
        \skipln$\qquad\qquad\qquad\qquad\qquad\qquad\left.\vphantom{\edgeQueue \setInsert \pair{\statev \in \vertexSet}{\statew \in \nearVertices} \pair{\statev}{\statew} \not\in \edgeSet,\;\gBelow{\statev} + \cBelow{\statev}{\statew} + \hBelow{\statew} < \gAbove{\xgoal}}\gAbove{\statev} + \cBelow{\statev}{\statew} < \gAbove{\statew} \right\rbrace$\;
    }
\end{algorithm}%
\begin{algorithm}[t]
    \caption{Prune\algorithmStyle$\left(\cost\in\mathbb\mathbb{R}_{\geq0}\right)$}\label{algo:prune}
    \algorithmStyle
    $\sampleSet \setRemove \setst{\statex\in\sampleSet}{\fBelow{\statex}\geq\cost}$\label{algo:prune:pruneSamples}\;
    $\vertexSet \setRemove \setst{\statev\in\vertexSet}{\fBelow{\statev} > \cost}$\label{algo:prune:rmBadVertices}\;
    $\edgeSet \setRemove \setst{\pair{\statev}{\statew}\in\edgeSet}{\fBelow{\statev} > \cost,\;\mathtt{or}\;\fBelow{\statew} > \cost}$\label{algo:prune:rmBadEdges}\;
    $\sampleSet \setInsert \setst{\statev\in\vertexSet}{\gAbove{\statev} \equiv \infty}$\label{algo:prune:reuseVertices}\;
    $\vertexSet \setRemove \setst{\statev\in\vertexSet}{\gAbove{\statev} \equiv \infty}$\label{algo:prune:rmOrphanedVertices}\;
\end{algorithm}%
\begin{figure*}[t]
    \centering
    \includegraphics[width=\textwidth]{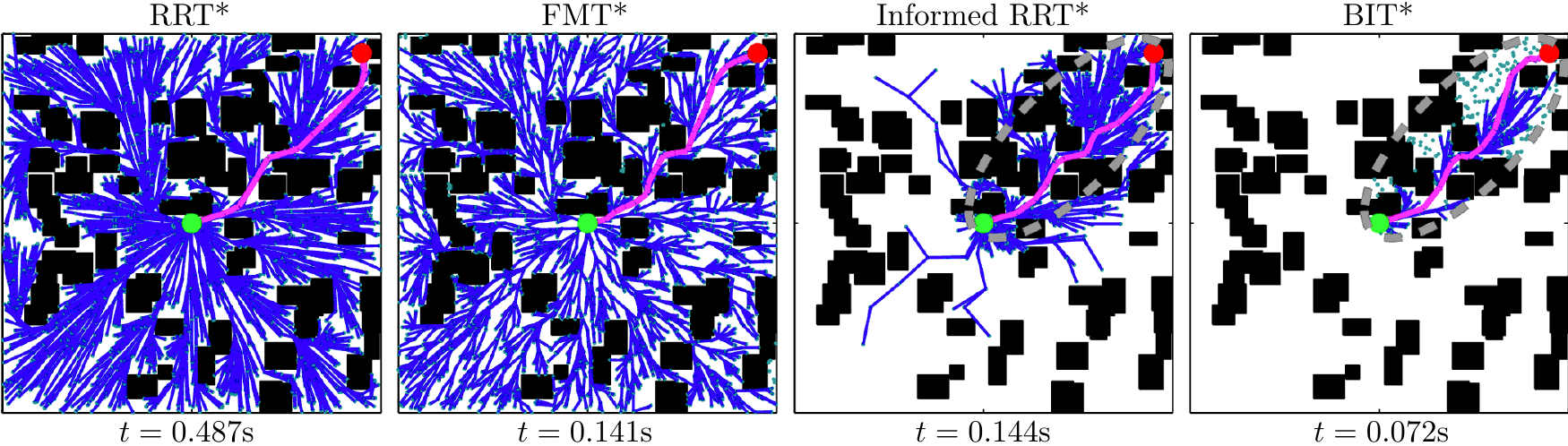}
    \caption{An example of \acs{RRTstar}, Informed \acs{RRTstar}, \ac{FMTstar} ($\samplesPerBatch = 2500$), and \acs{BITstar} run on a random $\mathbb{R}^2$ world.
    Each algorithm was run until it found a equivalent solution to \ac{FMTstar} ($c = 1.39$) regardless of homotopy class.
    \acs{BITstar}'s use of heuristics allows it to find such a solution faster ($t = 0.072$s) than \ac{RRTstar} ($t = 0.487$s), \ac{FMTstar} ($t = 0.141$s) and Informed \acs{RRTstar} ($t = 0.144$s) by performing its search in a principled manner that initially investigates low-cost solutions and focuses the search for improvements.
    Animated results are available in the attached video.
    \fullFigGap\vspace{-1ex}} %
    \label{fig:randComparison}
\end{figure*}%

\subsubsection{Edge selection (\algolines{algo:bitstar}{algo:bitstar:startExpand}{algo:bitstar:popEnd})} \label{sec:bit:detail:getEdge}
The tree is built by processing the queue of edges, $\edgeQueue$, in order of increasing estimated cost of a solution constrained to pass through the edge, $\pair{\statev}{\statex}$, given the current tree, $\gAbove{\statev} + \cBelow{\statev}{\statex} + \hBelow{\statex}$.
Ties are broken in favour of the edge with the lowest current cost-to-come to the source vertex, $\gAbove{\statev}$.
The function $\mathtt{BestInQueue}\left(\edgeQueue\right)$ returns the best edge in the queue given this ordering.
The function $\mathtt{BestQueueValue}\left(\edgeQueue\right)$ returns the estimated solution cost of the best edge in the queue.

The cost of creating the edge queue is delayed by using a vertex expansion queue, $\vertexQueue$.
This vertex queue is ordered on the estimated cost of a solution constrained to pass through the vertex given the current tree, $\gAbove{\statev} + \hBelow{\statev}$.
This value is a lower bound estimate of the edge-queue values from a vertex; therefore, vertices only need to be expanded into the edge queue when their vertex-queue value is less than the best edge-queue value.
The function $\mathtt{BestInQueue}\left(\vertexQueue\right)$ returns the best vertex in the vertex queue given this ordering.
The function $\mathtt{BestQueueValue}\left(\vertexQueue\right)$ returns the estimated solution cost of the best vertex in the queue.

Before selecting the next edge in the queue to process, any vertices that could have a better outgoing edge (\algoline{algo:bitstar}{algo:bitstar:startExpand}) are expanded (\algoline{algo:bitstar}{algo:bitstar:updateEdgeQueue}; \algo{algo:expandVertex}).
The best edge in the queue, $\bestEdge$, is then removed for processing (\algolines{algo:bitstar}{algo:bitstar:popStart}{algo:bitstar:popEnd}).
As edges are only added to the edge queue by expanding their source vertex, and each vertex is only expanded once per batch, each edge is guaranteed to only be processed once per batch.

\subsubsection{Edge processing (\algolines{algo:bitstar}{algo:bitstar:startProcess}{algo:bitstar:endProcess})} \label{sec:bit:addEdge}
Heuristics are used to accelerate the processing of edges and delay the calculation of the true edge cost.
The edge being processed, $\bestEdge$, is first checked to see if it can improve the current solution given the current tree (\algoline{algo:bitstar}{algo:bitstar:canBeBetterSoln}).
If it cannot, then by construction no other edges in the queue can and both queues are cleared to start a new batch (\algoline{algo:bitstar}{algo:bitstar:clearQueue}).

The true edge cost is then calculated by performing collision checks and solving any differential constraints.
This may be expensive, so the edge is processed if it could \emph{ever} improve the current solution, regardless of the current state of the tree (\algoline{algo:bitstar}{algo:bitstar:canRealEdgeEverBeBetterSoln}).
If it cannot, than it is discarded.

Finally, the edge is checked to see if it improves the cost-to-come of its target vertex (\algoline{algo:bitstar}{algo:bitstar:improvesGraph}), noting that disconnected vertices have an infinite cost.
If it does, it is added to the tree.

If the target vertex, $\xmin$, is in the tree (\algoline{algo:bitstar}{algo:bitstar:isInTree}), then the edge represents a \emph{rewiring}, otherwise it is an \emph{expansion}.
Rewirings require removing the edge to the target vertex from the tree (\algoline{algo:bitstar}{algo:bitstar:rmParent}).
Expansions require moving the target vertex from the set of unconnected samples to the set of vertices and queueing it for expansion (\algolines{algo:bitstar}{algo:bitstar:rmSample}{algo:bitstar:addVertex}).

The new edge is then added to the tree (\algoline{algo:bitstar}{algo:bitstar:addEdge}) and the edge queue is pruned to remove edges that cannot improve the cost-to-come of the vertex (\algoline{algo:bitstar}{algo:bitstar:pruneEdges}).

\subsubsection{Vertex Expansion (\algo{algo:expandVertex})}\label{sec:bit:detail:expandVertex}
The function, $\mathtt{Expand}$-$\mathtt{Vertex}\left(\statev\right)$, removes a vertex, $\statev\in\vertexQueue\subseteq\vertexSet$, from the vertex queue (\algoline{algo:expandVertex}{algo:expandVertex:rmVertexQueue}) and adds outgoing edges from the vertex to the edge queue.

In the \ac{RGG}, a vertex is connected to all states within a radius, $r$.
Edges to unconnected states (\algoline{algo:expandVertex}{algo:expandVertex:nearSamples}) are always added to edge queue if they could be part of a better solution (\algoline{algo:expandVertex}{algo:expandVertex:insert}).
Edges to connected states are only added if the source vertex was added to the tree during this batch (\algoline{algo:expandVertex}{algo:expandVertex:isNew}).
This prevents repeatedly checking edges between vertices in the tree.
These rewiring edges (\algoline{algo:expandVertex}{algo:expandVertex:nearVertices}) are added to the edge queue if, in addition to possibly providing a better solution, they are not already in the tree and could improve the path to the target vertex given the current tree (\algoline{algo:expandVertex}{algo:expandVertex:insertRewire}).

\subsubsection{Graph Pruning (\algo{algo:prune})}\label{sec:bit:detail:prune}
The function, $\mathtt{Prune}\left(\cost\right)$, removes states that cannot provide a solution better than the given cost, $\cost\in\mathbb{R}_{\geq0}$.
Unconnected samples are removed (\algoline{algo:prune}{algo:prune:pruneSamples}), while vertices in the tree are removed and disconnected (\algolines{algo:prune}{algo:prune:rmBadVertices}{algo:prune:rmBadEdges}).
To maintain uniform sample density in the subproblem being searched, disconnected descendents that could still provide a better solution are returned to the unconnected sample set (\algolines{algo:prune}{algo:prune:reuseVertices}{algo:prune:rmOrphanedVertices}).

\subsection{Practical Considerations}\label{sec:bit:pract}
\vspace{-0.25ex}
Algs.~\ref{algo:bitstar}--\ref{algo:prune} describe \ac{BITstar} without considering implementation, leaving room for practical improvements.
Pruning (\algoline{algo:bitstar}{algo:bitstar:prune}) is expensive and should only occur when a new solution has been found.
It can even be limited to \emph{significant} changes in solution cost without altering behaviour.

Searches (e.g., \algoline{algo:bitstar}{algo:bitstar:rmParent}; \algoline{algo:expandVertex}{algo:expandVertex:nearSamples}; \algoline{algo:prune}{algo:prune:rmBadEdges}; etc.) can be implemented efficiently with appropriate datastructures, e.g., $k$-d trees or indexed containers, that do not require an exhaustive global search.

Ordered containers provide an efficient edge queue (\algolines{algo:bitstar}{algo:bitstar:popStart}{algo:bitstar:popEnd}).
While rewirings will change the order of some elements, we found little experimental difference between an approximately sorted and a strictly sorted queue.

\begin{figure*}[t]
    \centering
    \includegraphics[width=\textwidth]{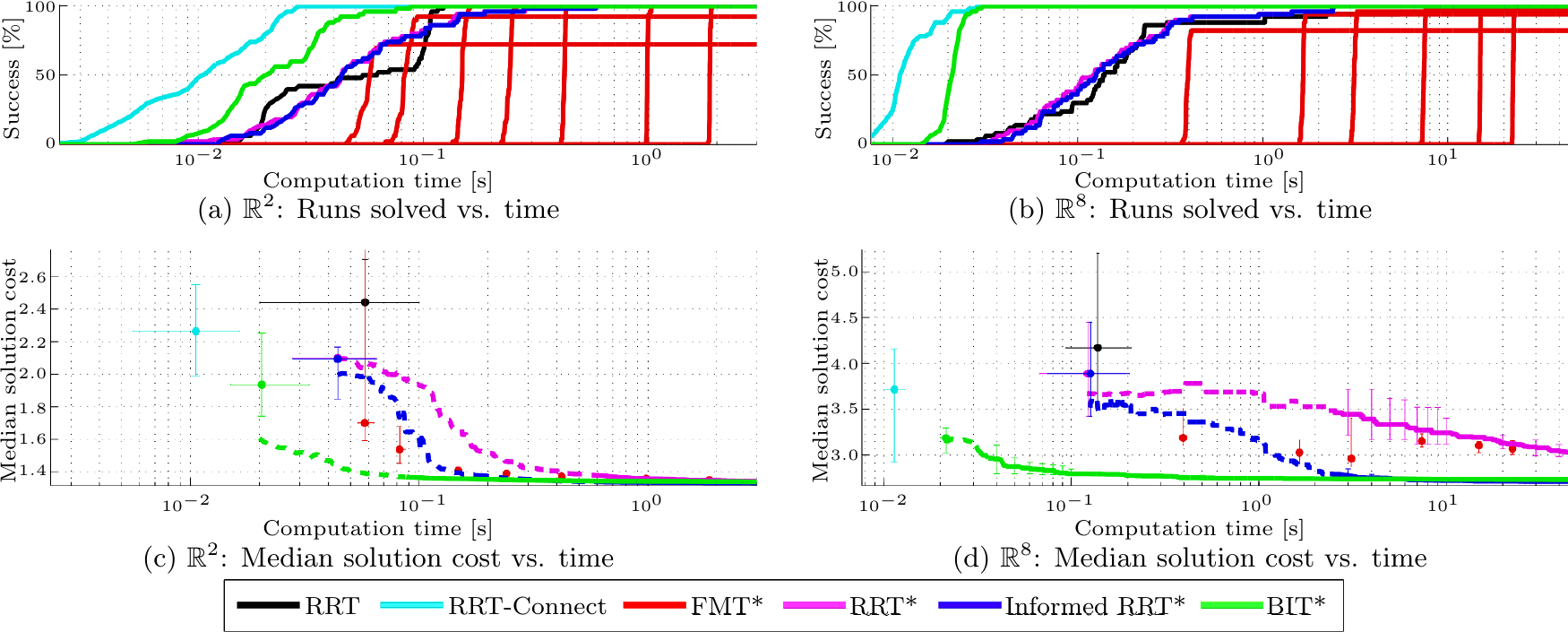}
    \caption{The results from representative worlds in $\mathbb{R}^2$ and $\mathbb{R}^8$ for \acs{RRT}, \acs{RRT}-Connect, \acs{RRTstar}, Informed \acs{RRTstar}, \ac{BITstar} with a batch size of $100$ samples, and \ac{FMTstar} of various sample sizes ($\mathbb{R}^2$: $500$, $1000$, $2500$, $5000$, $10000$, $25000$, and $50000$; $\mathbb{R}^8$: $100$, $500$, $1000$, $2500$, $5000$, and $7500$).
    For the chosen random worlds, (a) and (b) show the  percentage of trials solved versus run time for the $50$ different trials, while (c) and (d) show the median solution cost versus run time.
    Dots represent the median initial solution.
    For algorithms that asymptotically converge towards the optimum, the dashed lines represent a median calculated from $50\%$--$100\%$ success rate and may increase as new trials are included.
    The solid lines represent the median when all trials have a solution, with error bars denote a non-parametric $95\%$ confidence interval on median solution cost and time.
    Note that for some algorithms the confidence intervals are smaller than the median line and are not visible and that \acs{RRT} and \acs{RRT}-Connect are not asymptotically optimal planners.
    \fullFigGap\vspace{-0.95ex}}
    \label{fig:runHistories}
\end{figure*}%
\section{Analysis}\label{sec:anal}
For brevity, we only present a proof of almost sure asymptotic optimality (Theorem~\ref{thm:ao}) and note that this implies probabilistic completeness.
We also present a discussion on the relationship between \ac{BITstar}'s edge queue and \ac{LPAstar}'s vertex queue (Remark~\ref{rem:lpa}).

\begin{thm}[Asymptotic Optimality]\label{thm:ao}
    \ac{BITstar} asymptotically converges \emph{almost surely} to the optimal solution to Prob.~\ref{pdef:opt}, if a solution exists, as the total number of samples, $\totalSamples$, goes to infinity, i.e.,%
    \setlength \abovedisplayskip{-1.0ex plus0pt minus1pt}%
    \setlength \belowdisplayskip{0.5ex plus0pt minus1pt}%
    \begin{align*}
        \mathStyle
        \prob{\limsup\limits_{\totalSamples\to\infty}\cbestfinal^{\rm \scriptscriptstyle BIT*} = \bestPathCost} = 1,
    \end{align*}%
    \setlength \abovedisplayskip{1ex plus0pt minus1pt}%
    \setlength \belowdisplayskip{1ex plus0pt minus1pt}%
    where $\cbestfinal^{\rm \scriptscriptstyle BIT*}$ is the cost of the best solution found by \ac{BITstar} from $\totalSamples$ samples.
\end{thm}%
\begin{proof}
    The proof extends directly from the work in \cite{karaman_ijrr11}.
    In Appendix G, Karaman and Frazzoli show that for $\totalSamples$ uniformly distributed random samples and a specific \emph{constant} $r_\totalSamples$, the solution found by \ac{RRTstar} almost surely converges asymptotically to the optimal solution as $\totalSamples$ goes to infinity,~i.e.,%
    \setlength \abovedisplayskip{0.7ex plus0pt minus1pt}%
    \setlength \belowdisplayskip{0.7ex plus0pt minus1pt}%
    \begin{align*}
        \mathStyle
        \prob{\limsup\limits_{\totalSamples\to\infty}\cbestfinal^{\rm \scriptscriptstyle RRT*} = \bestPathCost} = 1.
    \end{align*}%
    \setlength \abovedisplayskip{1ex plus0pt minus1pt}%
    \setlength \belowdisplayskip{1ex plus0pt minus1pt}%
    
    \ac{RRTstar} processes the sequence of $q$ samples individually.
    For any sample, it considers all edges involving samples earlier in the sequence that are less than length $r_\totalSamples$.
    \ac{BITstar} processes the sequence of samples in batches.
    For any sample in a batch, it considers all edges involving samples from the same or earlier batches that are less than length $r_\totalSamples$.
    This will contain all the edges considered by \ac{RRTstar} for the same sequence and $r_\totalSamples$.
    As \ac{BITstar} maintains uniform sample density in the subproblem that contains all better solutions and \eqref{eqn:r} meets the requirements for almost sure asymptotic optimality given in \cite{karaman_ijrr11}, \ac{BITstar} is almost surely asymptotically optimal.
\end{proof}%

\begin{rem}[Equivalence to \acs{LPAstar} vertex queue]\label{rem:lpa}
    \ac{BITstar}'s edge queue is an extension of \ac{LPAstar}'s vertex queue \cite{koenig_ai04} to include a heuristic estimate of edge cost.
\end{rem}%
\begin{expl}
    \ac{LPAstar} uses a queue of vertices ordered lexicographically first on the solution cost constrained to go through the vertex and then the cost-to-come to the vertex.
    Both these terms are calculated for a vertex, $\statev \in \vertexSet$, considering all the incoming edges (\textit{rhs-value} in \ac{LPAstar}), i.e.,%
    \setlength \abovedisplayskip{0.6ex plus0pt minus1pt}%
    \setlength \belowdisplayskip{0.6ex plus0pt minus1pt}%
    \begin{align}\label{eqn:rhs}
        \mathStyle
        \min\limits_{\pair{\stateu}{\statev}\in\edgeSet}\left\lbrace\gAbove{\stateu}+\cTrue{\stateu}{\statev}\right\rbrace,
    \end{align}%
    where $\edgeSet$ is the set of edges.%
    \setlength \abovedisplayskip{1ex plus0pt minus1pt}%
    \setlength \belowdisplayskip{1ex plus0pt minus1pt}%

    This minimum requires the calculation of the true edge cost between a vertex and all of its possible parents.
    This calculation is expensive in sampling-based planning (e.g., collision checking, differential constraints, etc.), and reducing its calculation is desirable.
    This can be done by using an admissible heuristic estimate of edge cost and calculating \eqref{eqn:rhs} incrementally.
    A running minimum is calculated by processing edges in order of increasing \emph{estimated} cost.
    The process finishes, and the true minimum is found, when the estimated cost through the next edge is higher than the current value.
    
    \ac{BITstar} combines these individual minima calculations into a single edge queue.
    In doing so, it simultaneously calculates the minimum cost-to-come for each vertex while expanding vertices in order of increasing estimated solution cost.
\end{expl}%

\section{Experimental Results}\label{sec:exp}
\ac{BITstar} was tested against existing algorithms in both simulated random worlds (Section \ref{sec:exp:sim}) and real-world manipulation problems (Section \ref{sec:exp:herb}) using publicly available \ac{OMPL} \cite{ompl} implementations.
All tests and algorithms used an \ac{RGG} constant (e.g., $\eta$ in \eqref{eqn:r}) of $1.1$ and approximated $\lebesgue{\freeSet}$ with $\lebesgue{\stateSet}$.
\ac{RRT}-based algorithms used a goal bias of $5\%$.
\ac{BITstar} used $100$ samples per batch, Euclidean distance between states for heuristics, and direct informed sampling \cite{gammell_iros14}.
Graph pruning was limited to changes in the solution cost greater than $1\%$ and we used an approximately sorted queue.

\subsection{Simulated Random Worlds}\label{sec:exp:sim}

\ac{BITstar} was compared to existing sampling-based algorithms on random problems minimizing path length in $\mathbb{R}^{2}$ and $\mathbb{R}^{8}$.
The problems consisted of a (hyper)cube of width $2$ populated with random axis-aligned (hyper)rectangular obstacles such that at most one third of the environment was obstructed.
The initial state was in the centre of the world and the goal was $(0.9, 0.9, \ldots, 0.9)$ away (Fig.~\ref{fig:randComparison}).
\ac{BITstar} was compared to the \ac{OMPL} implementations of \ac{RRT}, \acs{RRT}-Connect \cite{kuffner_icra00}, \ac{RRTstar}, Informed \acs{RRTstar}, and \ac{FMTstar}.
The \ac{RRT}-based planners used a maximum edge length of $0.2$ and $1.25$ in $\mathbb{R}^2$ and $\mathbb{R}^8$, respectively.
All algorithm parameters were chosen in good faith to maximize performance on a separate training set of random worlds.

For each state dimension, $10$ different random worlds were generated and the planners were tested with $50$ different pseudo-random seeds on each.
The solution cost of each planner was recorded every $1$~millisecond by a separate thread\footnotemark{}.
For each world, median solution cost was calculated for a planner by interpolating each trial at a period of $1$~millisecond.
As the true optima for these problems are different and unknown, there is no meaningful way to compare the results across problems.
Instead, results from a representative problem are presented in Fig.~\ref{fig:runHistories}, where the percent of trials solved and the median solution cost are plotted versus computational time.%
\footnotetext{Simulations were run on a MacBook Pro with $4$~GB of RAM and an Intel i7-620M processor running a $64$-bit version of Ubuntu 12.04.}

These experiments show that in both $\mathbb{R}^2$ (Figs.~\ref{fig:runHistories}a, \ref{fig:runHistories}c) and $\mathbb{R}^8$ (Figs.~\ref{fig:runHistories}b, \ref{fig:runHistories}d), \ac{BITstar} generally finds better solutions faster than other sampling-based optimal planners and \ac{RRT}.
It has a higher likelihood of having found a solution at a given computational time than these planners, and converges faster towards the optimum.
The only planner tested that found solutions faster than \ac{BITstar} was \acs{RRT}-Connect, a nonasymptotically optimal planner.

\subsection{Motion Planning for Manipulation}\label{sec:exp:herb}
To evaluate the performance of \ac{BITstar} on real-world high-dimensional problems, it was tested on \ac{HERB} \cite{herb}.
Experiments consisted of both dual-arm and one-arm planning problems for manipulation with a goal of minimizing the path length through configuration space.
Parameter values for \ac{BITstar} and \ac{RRT}-based planners were chosen from the results of Section~\ref{sec:exp:sim}, and the number of \ac{FMTstar} samples was chosen to use the majority of the available computational time.
Once again, \ac{BITstar} outperformed all planners other than \acs{RRT}-Connect.

For the dual-arm planning problem, \ac{HERB} started with both arms extended under a table from the elbow onward.
The task was to plan a trajectory for both arms to place the hands in position to open a bottle  (Fig.~\ref{fig:2armHerb}).
\acs{HERB}'s proximity to the table and starting position created a narrow passage for the arms around the table.
Coupled with the $14$-\ac{DOF} configuration space, this made for a challenging problem.

Given $2.5$ minutes\footnotemark{} of planning time, \ac{BITstar} was almost twice as likely to find a solution than \ac{RRT}, Informed \acs{RRTstar}, or \ac{FMTstar}.
Over $25$ trials, \acs{BITstar} was $68\%$ successful with a median solution cost of $17.4$.
\ac{RRT}-Connect was $100\%$ successful, but had a median solution cost of $22.1$.
\ac{RRT} was $8\%$ successful with a median solution cost of $31.1$ and Informed \acs{RRTstar} was $8\%$ successful with a median solution cost of $25.3$.
\acs{FMTstar} with $m=500$ was $36\%$ successful with a median solution cost of $17.2$.
All \acs{RRT}-based planners used a maximum edge length of $3$.%
\footnotetext{\acs{HERB} experiments were run on a Dell T3500 with $12$~GB of RAM and an Intel W3565 processor running a $64$-bit version of Ubuntu 12.04.}

An easier one-arm planning problem was also tested.
\ac{HERB} started with its left arm folded at the elbow and held at approximately the table level of a table.
The task was to plan a trajectory to place the left hand in position to grasp a box (Fig.~\ref{fig:1armHerb}).
The smaller configuration space, $7$ \ac{DOF}, and a starting position partially clear of the table made this an easier planning problem.
In the given $5$~seconds of computational time, both \ac{BITstar} and \ac{RRT}-Connect found a solution in all $25$ trials.
\acs{BITstar} had a median solution cost of $6.8$ while \ac{RRT}-Connect had a median solution cost of $10.6$.
\ac{RRT} was $88\%$ successful with a median solution cost of $11.2$ and Informed \acs{RRTstar} was $88\%$ successful with a median solution cost of $10.6$.
\acs{FMTstar} with $m=50$ was $52\%$ successful with a median solution cost of $9.0$.
All \acs{RRT}-based planners used a a maximum edge length of $1.25$.

\addtolength{\textheight}{-1.1ex}

\begin{figure}[t]
    \centering
    \includegraphics[width=\columnwidth]{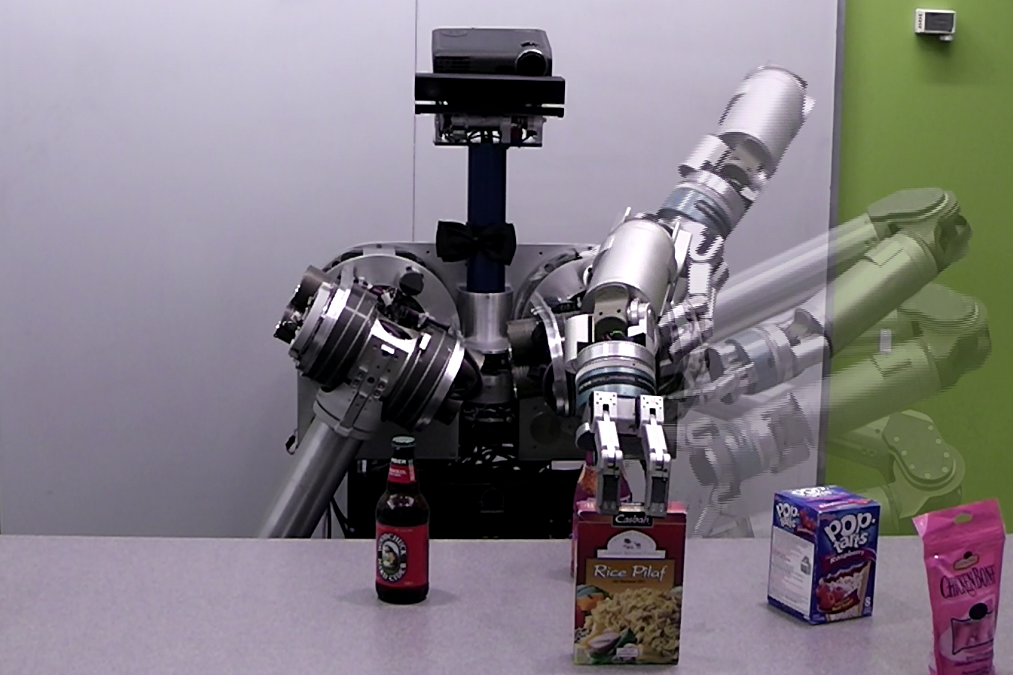}
    \caption{A composite figure of a one-arm trajectory on \acs{HERB} found by \acs{BITstar}.
    Over $25$ trials with $5$ seconds of computational time, \acs{BITstar} had a median solution cost of $6.8$ and success rate of $100\%$, while Informed \acs{RRTstar} and \acs{FMTstar} had median costs of $10.6$ and $9.0$ and success rates of $88\%$ and $52\%$, respectively.
    Nonoptimal planners, \acs{RRT} and \acs{RRT}-Connect, had median costs of $11.2$ and $10.6$ and success rates of $88\%$ and $100\%$,~respectively.}
    \label{fig:1armHerb}
\end{figure}%
\section{Discussion \& Future Work}\label{sec:disc}
\ac{BITstar} demonstrates that anytime sampling-based planners can be designed by combining incremental graph-search techniques with \ac{RGG} theory.
We hope that this work will motivate further unification of these two planning paradigms.

A fundamental component of \ac{BITstar} is the application of heuristic estimates to \emph{all} aspects of path cost.
Doing so allows the algorithm to account for future graph improvements (cost-to-come), avoid unnecessary collision checks and boundary-value problems (edge cost), and order and focus the search (solution cost).
As always, the benefit of these heuristics will depend on their suitability for the specific problem, but we feel that they are an important tool to reduce the \textit{curse of dimensionality}.
Note that while direct sampling of the subproblem is possible for some cost functions \cite{gammell_iros14}, rejection sampling is applicable.
Also note that, as with other heuristically guided searches (e.g., A*), \ac{BITstar} works with the trivial \emph{zero} heuristic (e.g., Dijkstra's algorithm); however, more conservative heuristics provide less benefit to the search.

In describing \ac{BITstar} as an extension of \ac{LPAstar} to continuous planning problems, it is important to note a key difference in how they reuse information.
In \ac{LPAstar}, updating the cost-to-come of a vertex requires reconsidering the cost-to-come of all possibly descendent vertices.
This is a step that becomes prohibitively expensive in anytime resolution planners as graph size increases quickly.
The results of \ac{RRTstar} demonstrate that this is unnecessary for the planner to almost surely converge asymptotically to the optimum as the number of samples approaches infinity.

While the efficiency of graph-search techniques is well understood, this area remains understudied for sampling-based planners.
We are actively investigating whether \ac{BITstar}'s use of graph-search techniques and \ac{RGG} theory can be used to probabilistically evaluate its efficiency.

Also of interest are possible improvements to \ac{BITstar}, including the fact that \ac{BITstar} does not remove samples when connection attempts fail.
This is a requirement of the uniform sample distribution used in \ac{RGG} theory, but leaves edges in the implicit \ac{RGG} that are known to be unusable.

Finding an efficient method to avoid these edges would improve \ac{BITstar}, and there are multiple potential ways to accomplish this.
Failed edges could be tracked and prevented from reentering the queue, but initial attempts have proven too computational expensive.
Samples that fail multiple connection attempts could be removed, but doing so will require \ac{RGG} theory for nonuniform distributions.
Our current focus is on the adaptively varying batch size to increase the rate at which these edges are removed from the \ac{RGG}.

We are also interested in more general extensions to \ac{BITstar}.
Its expanding search is well suited for large or unbounded planning problems, and we have had initial success with a version that generates samples as needed and avoids the \textit{a priori} definition of state space limits.
Its relationship to incremental search techniques also suggests it may be well suited for planning problems in changing environments.
We are also investigating the use of other graph-search techniques, including anytime \cite{likhachev_icaps05,ferguson_icra05,likhachev_ai08} or bidirectional \cite{pohl_mi71, sint_jacm77} searches to decrease the time required to find an initial solution.
Finally, we are investigating combining \ac{BITstar}'s global search with local searches, such as path-smoothing.

\section{Conclusion}\label{sec:conc}
In this paper, we attempt to unify graph-search and sampling-based planning techniques through \ac{RGG} theory.
By recognizing that a set of samples defines an implicit \ac{RGG} and using incremental-search techniques, we are able to combine the efficient search of algorithms such as A*, with the anytime scalability of sampling-based algorithms such as \ac{RRTstar}.
The resulting algorithm, \ac{BITstar}, uses heuristics for all aspects of path cost in order to prioritize the search of high-quality paths and focus the search for improvements.

As demonstrated on both simulated and real-world experiments, \ac{BITstar} outperforms existing sampling-based optimal planners and \ac{RRT}, especially in high dimensions.
For a given computational time, \ac{BITstar} has a higher likelihood of finding a solution and generally finds solutions of equivalent quality sooner.
It also converges towards the optimum faster than other asymptotic optimal planners, and has recently been shown to perform well on problems with differential constraints \cite{xie_icra15}.
Information on the \ac{OMPL} implementation of \ac{BITstar} is available at {\footnotesize\url{http://asrl.utias.utoronto.ca/code}}.

\begin{spacing}{0.9}%
\section*{Acknowledgment}
We would like to thank Christopher Dellin, Michael Koval, and Rachel Holladay for their comments on drafts of this work and Jennifer King for her help with experiments on \ac{HERB}.
This research was funded by contributions from the \ac{NSERC} through the \ac{NCFRN}, the Ontario Ministry of Research and Innovation's Early Researcher Award Program, and the \ac{ONR} Young Investigator Program.
\end{spacing}

\vspace{-1sp}%
\begin{spacing}{0.88}%

\end{spacing}


\begin{thebibliography}{10}
\vspace{-0.5ex}
\providecommand{\url}[1]{#1}
\csname url@rmstyle\endcsname
\providecommand{\newblock}{\relax}
\providecommand{\bibinfo}[2]{#2}
\providecommand\BIBentrySTDinterwordspacing{\spaceskip=0pt\relax}
\providecommand\BIBentryALTinterwordstretchfactor{4}
\providecommand\BIBentryALTinterwordspacing{\spaceskip=\fontdimen2\font plus
\BIBentryALTinterwordstretchfactor\fontdimen3\font minus
  \fontdimen4\font\relax}
\providecommand\BIBforeignlanguage[2]{{%
\expandafter\ifx\csname l@#1\endcsname\relax
\typeout{** WARNING: IEEEtran.bst: No hyphenation pattern has been}%
\typeout{** loaded for the language `#1'. Using the pattern for}%
\typeout{** the default language instead.}%
\else
\language=\csname l@#1\endcsname
\fi
#2}}
\bibitem{dijkstra_59}
E.~W.~Dijkstra, ``A note on two problems in connexion with graphs,''
  \emph{Numerische Mathematik}, 1(1): 269--271, 1959.

\bibitem{hart_tssc68}
P.~E.~Hart, N.~J.~Nilsson, and B.~Raphael, ``A formal basis for the heuristic
  determination of minimum cost paths,'' \emph{TSSC}, 4(2): 100--107, Jul.~1968
%
%

\bibitem{bellman_ams54}
R.~E.~Bellman, ``The theory of dynamic programming,'' \emph{Bull.\ of the
  AMS}, 60(6): 503--516, 1954.

%
%
%
%

\bibitem{bertsekas_tac75}
D.~P.~Bertsekas, ``Convergence of discretization procedures in dynamic programming,''
 \emph{TAC}, 20(3): 415--419, Jun.~1975.
%

\bibitem{bellman_57}
R.~E.~Bellman, \emph{Dynamic Programming}. Princeton Uni.\ Press, 1957.

\bibitem{lozano-perez_cacm79}
T.~Lozano-P{\'e}rez and M.~A.~Wesley, ``An algorithm for planning
  collision-free paths among polyhedral obstacles,'' \emph{CACM},
  22(10): 560--570, Oct.~1979.
%

\bibitem{chen_icra92}
P.~C.~Chen and Y.~K.~Hwang, ``{SANDROS}: a motion planner with performance
proportional to task difficulty,'' in \emph{ICRA}, 3: 2346--2353, May~1992.

\bibitem{sallaberger_acta95}
C.~S.~Sallaberger and G.~M.~D'Eleuterio, ``Optimal robotic path planning using
dynamic programming and randomization,'' \emph{Acta Astronautica}, 35(2--3): 143--156, 1995.

\bibitem{barraquand_icra91}
J.~Barraquand and J.-C.~Latombe, ``Nonholonomic multibody mobile robots:
  controllability and motion planning in the presence of obstacles,'' in
  \emph{ICRA}, 3: 2328--2335, Apr.~1991.

\bibitem{lynch_ijrr96}
K.~M.~Lynch and M.~T.~Mason, ``Stable pushing: Mechanics, controllability, and
  planning,'' \emph{IJRR}, 15(6): 533--556, 1996.

\bibitem{cherif_icra99}
M.~Cherif, ``Kinodynamic motion planning for all-terrain wheeled vehicles,'' in
  \emph{ICRA}, 1: 317--322, 1999.

\bibitem{donald_jacm93}
B.~Donald, P.~Xavier, J.~Canny, and J.~Reif, ``Kinodynamic motion planning,''
  \emph{JACM}, 
  40(5): 1048--1066, Nov.~1993.
%

\bibitem{kondo_tra91}
K.~Kondo, ``Motion planning with six degrees of freedom by multistrategic
  bidirectional heuristic free-space enumeration,'' \emph{TRA},
  7(3): 267--277, Jun.~1991.

\bibitem{likhachev_icaps05}
M.~Likhachev, D.~Ferguson, G.~Gordon, A.~Stentz, and S.~Thrun, ``Anytime
dynamic {A*}: An anytime, replanning algorithm,'' in \emph{ICAPS}, Jun.~2005.

\bibitem{ferguson_icra05}
D.~Ferguson and A.~Stentz, ``The delayed {D*} algorithm for efficient path
replanning,'' in \emph{ICRA}, 2045--2050, Apr.~2005.

\bibitem{likhachev_ai08}
M.~Likhachev, D.~Ferguson, G.~Gordon, A.~Stentz, and S.~Thrun, ``Anytime search
in dynamic graphs,'' \emph{Art.\ Intel.}, 172(14): 1613--1643, 2008.

\bibitem{koenig_ai04}
S.~Koenig, M.~Likhachev, and D.~Furcy, ``Lifelong planning {A*},''
\emph{Art.\ Intel.}, 155(1--2): 93--146, 2004.
%

\bibitem{stentz_ijcai95}
A.~Stentz, ``The focussed {D*} algorithm for real-time replanning,'' in
\emph{IJCAI} 1652--1659, 1995.
%

\bibitem{koenig_tro05}
S.~Koenig and M.~Likhachev, ``Fast replanning for navigation in unknown
terrain,'' \emph{TRO}, 21(3): 354--363, Jun.\ 2005.

\bibitem{kavraki_tro96}
L.~E.~Kavraki, P.~\v{S}vestka, J.-C.~Latombe, and M.~H.~Overmars,
  ``Probabilistic roadmaps for path planning in high-dimensional configuration
  spaces,''  \emph{TRA}, 12(4): 566--580, 1996.
%

\bibitem{lavalle_ijrr01}
S.~M.~LaValle and J.~J.~{Kuffner Jr.}, ``Randomized kinodynamic planning,''
  \emph{IJRR}, 20(5): 378--400, 2001.
%

\bibitem{hsu_ijrr02}
D.~Hsu, R.~Kindel, J.-C.~Latombe, and S.~Rock, ``Randomized kinodynamic motion
  planning with moving obstacles,'' \emph{IJRR}, 21(3): 233--255, 2002.
%

\bibitem{karaman_ijrr11}
S.~Karaman and E.~Frazzoli, ``Sampling-based algorithms for optimal motion
planning,'' \emph{IJRR}, 30(7): 846--894, 2011.
%

\bibitem{urmson_iros03}
C.~Urmson and R.~Simmons, ``Approaches for heuristically biasing {RRT}
growth,'' \emph{IROS}, 2: 1178--1183, 2003.
%

\bibitem{ferguson_iros06}
D.~Ferguson and A.~Stentz, ``Anytime {RRT}s,'' \emph{IROS},
5369--5375, 2006.
%

\bibitem{akgun_iros11}
B.~Akgun and M.~Stilman, ``Sampling heuristics for optimal motion planning in
high dimensions,'' \emph{IROS}, 2640--2645, 2011.
%

\bibitem{otte_tro13}
M.~Otte and N.~Correll, ``{C-FOREST}: Parallel shortest path planning with
superlinear speedup,'' \emph{TRO}, 29(3): 798--806, Jun.~2013
%

\bibitem{gammell_iros14}
J.~D.~Gammell, S.~S.~Srinivasa, and T.~D.~Barfoot, ``Informed {RRT*}: Optimal
  sampling-based path planning focused via direct sampling of an admissible
  ellipsoidal heuristic,'' in \emph{IROS}, 2997--3004, 2014.
  
\bibitem{janson_isrr13}
L.~Janson and M.~Pavone, ``Fast marching trees: a fast marching sampling-based
method for optimal motion planning in many dimensions,'' in \emph{ISRR}, Dec.~2013.
%
  
\bibitem{salzman_icra15}
O.~Salzman and D.~Halperin, ``Asymptotically-optimal motion planning using
lower bounds on cost,'' in \emph{ICRA}, 2015.%
    
\bibitem{diankov_iros07}
R.~Diankov and J.~J.~{Kuffner Jr.}, ``Randomized statistical path planning,''
in \emph{IROS}, 2007.

\bibitem{persson_ijrr14}
S.~M.~Persson and I.~Sharf, ``Sampling-based {A*} algorithm for robot
path-planning,'' \emph{IJRR}, 33(13): 1683--1798, 2014.

\bibitem{herb}
S.~Srinivasa, D.~Berenson, M.~Cakmak, A.~{Collet Romea}, M.~Dogar, A.~Dragan,
  R.~A. Knepper, T.~D. Niemueller, K.~Strabala, J.~M. Vandeweghe, and
  J.~Ziegler, ``{HERB} 2.0: Lessons learned from developing a mobile
  manipulator for the home,'' \emph{Proc.\ IEEE},
  100(8): 1--19, Jul.~2012.

\bibitem{penrose_03}
M.~Penrose, \emph{Random Geometric Graphs}, ser. Oxford Studies in Probability,
L.~C.~G.~Rogers, Ed. Oxford Uni.\ Press, 5: 2003.

\bibitem{xue_wire04}
F.~Xue and P.~R.~Kumar, ``The number of neighbors needed for connectivity of wireless networks,''
\emph{Wireless Networks}, 10(2): 169--181, 2004.

\bibitem{gilbert_siam61}
E.~N.~Gilbert, ``Random plane networks,'' \emph{SIAM}, 9(4): 533--543, 1961.
%

\bibitem{muthukrishnan_siam05}
S.~Muthukrishnan and G.~Pandurangan, ``The bin-covering technique for
thresholding random geometric graph properties,'' in \emph{SODA}, 989--998, 2005.

\bibitem{ompl}
I.~A.~{\c{S}}ucan, M.~Moll, and L.~E.~Kavraki, ``The {O}pen {M}otion {P}lanning
  {L}ibrary,'' \emph{{IEEE} R\&A Mag.}, 19(4): 72--82, Dec.\ 2012.
%

\bibitem{kuffner_icra00}
J.~J.~{Kuffner Jr.} and S.~M.~LaValle, ``{RRT-Connect}: An efficient approach
to single-query path planning,'' in \emph{ICRA}, 995--1001, 2000.

\bibitem{pohl_mi71}
I.~Pohl, ``Bi-directional search,'' \emph{Mach.\ Intel.}, 6: 127--140, 1971.

\bibitem{sint_jacm77}
L.~Sint and D.~{de Champeaux}, ``An improved bidirectional heuristic search
  algorithm,'' \emph{JACM}, 24(2): 177--191, Apr.\ 1977.

\bibitem{xie_icra15}
C.~Xie, J.~{van den Berg}, S.~Patil, and P.~Abbeel, ``Toward asymptotically
optimal motion planning for kinodynamic systems using a two-point boundary
value problem solver,'' in \emph{ICRA}, 2015.%

\end{thebibliography}
\end{document}